\pdfoutput=1

\documentclass[11pt]{article}

\usepackage{acl}

\usepackage{times}
\usepackage{latexsym}
\usepackage{graphicx}
\graphicspath{{images/}}

\usepackage[T1]{fontenc}

\usepackage[utf8]{inputenc}

\usepackage{microtype}

\title{IsoScore: Measuring the Uniformity of Embedding Space Utilization}

\author{William Rudman${}^\dagger$, 
Nate Gillman${}^\ddagger$,
Taylor Rayne${}^*$,
Carsten Eickhoff${}^\dagger$ \\
Department of Computer Science, Brown University${}^\dagger$\\ 
Department of Mathematics, Brown University${}^\ddagger$\\ 
Quest University${}^*$\\
\texttt{\{william\_rudman, ngillman, carsten\}@brown.edu}\\
\texttt{taylor.rayne@questu.ca}
}

\definecolor{green}{RGB}{10, 115, 10}
\definecolor{blue}{RGB}{100, 150, 240}
\definecolor{coral}{RGB}{255, 102, 102}
\definecolor{thistle}{RGB}{190,151,190}

\definecolor{green1}{RGB}{0, 148, 0}
\definecolor{red1}{RGB}{180, 0, 0}

\newcommand{\green}[1]{{\color{green1}{{\bf  }}{\em #1}{\bf }}}
\newcommand{\red}[1]{{\color{red1}{{\bf  }}{\em #1}{\bf }}}

\usepackage{amsmath}
\usepackage{amsfonts}
\usepackage{latexsym}
\usepackage{enumerate}
\usepackage{wasysym}
\usepackage{lipsum}
\usepackage{amssymb}
\usepackage{amsthm}
\usepackage{soul}

\newtheorem{definition}{Definition}
\numberwithin{definition}{section}
\newtheorem{prop}[definition]{Proposition}
\newtheorem{heuristic}[definition]{Heuristic}
\newtheorem*{assumptionUnderpinningHeuristic}{Assumption Underpinning The Heuristic}

\usepackage{pifont}
\newcommand{\cmark}{\ding{51}}%
\newcommand{\xmark}{\ding{55}}%
\usepackage{graphicx}
\graphicspath{ {arXiv_v1/images/} }

\newenvironment{psmallmatrix}
  {\left(\begin{smallmatrix}}
  {\end{smallmatrix}\right)}

\usepackage{algpseudocode,algorithm}
\algdef{SE}{Begin}{End}{\textbf{begin}}{\textbf{end}}

\usepackage{hyperref}
\DeclareMathOperator{\Diag}{Diag}
\DeclareMathOperator{\diag}{diag}

\numberwithin{equation}{section}

\usepackage{soul}

\begin{document}
\maketitle
\begin{abstract}
The recent success of distributed word representations has led to an increased interest in analyzing the properties of their spatial distribution. Several studies have suggested that contextualized word embedding models do not isotropically project tokens into vector space. However, current methods designed to measure isotropy, such as average random cosine similarity and the partition score, have not been thoroughly analyzed and are not appropriate for measuring isotropy.
We propose IsoScore: a novel tool that quantifies the degree to which a point cloud uniformly utilizes the ambient vector space. Using rigorously designed tests, we demonstrate that IsoScore is the only tool available in the literature that accurately measures how uniformly distributed variance is across dimensions in vector space.  Additionally, we use IsoScore to challenge a number of recent conclusions in the NLP literature that have been derived using brittle metrics of isotropy. 
We caution future studies from using existing tools to measure isotropy in contextualized embedding space as resulting conclusions will be misleading or altogether inaccurate.
\end{abstract}

\section{Introduction \& Background}\label{sec:intro}

The first step in any natural language processing pipeline is to represent text in a vector space. Understanding how contextualized word embedding models project tokens into vector space is crucial for advancing the field of natural language processing. Several recent studies analyzing the spatial distribution of contextualized word embeddings claim that the point clouds induced by models such as BERT or GPT-2 do not uniformly utilize all dimensions of the vector space they occupy \citep{bert_context, bert_mean, cai_isotropy_context, coenen_bert_geometry, gao2019representation}.

Figure~\ref{fig:dim_used_pics} illustrates a two-dimensional disk that uniformly utilizes the $x$ and $y$ axes in two-dimensional space, but does not uniformly utilize all dimensions when embedded into three dimensions.

 \begin{figure}[h!]
     \centering
     \includegraphics[width=2cm]{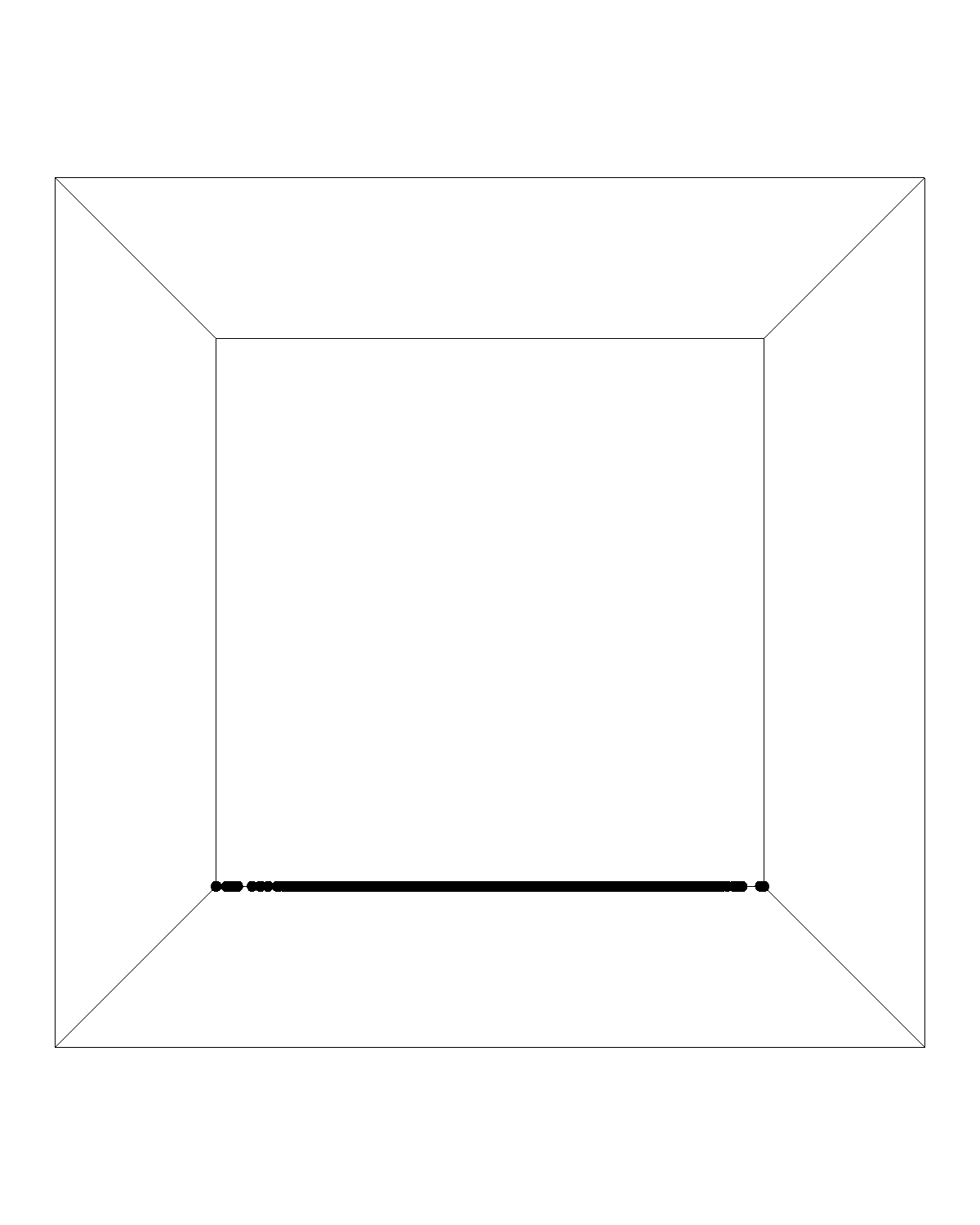}
     \hspace{0.5cm}
     \includegraphics[width=2cm]{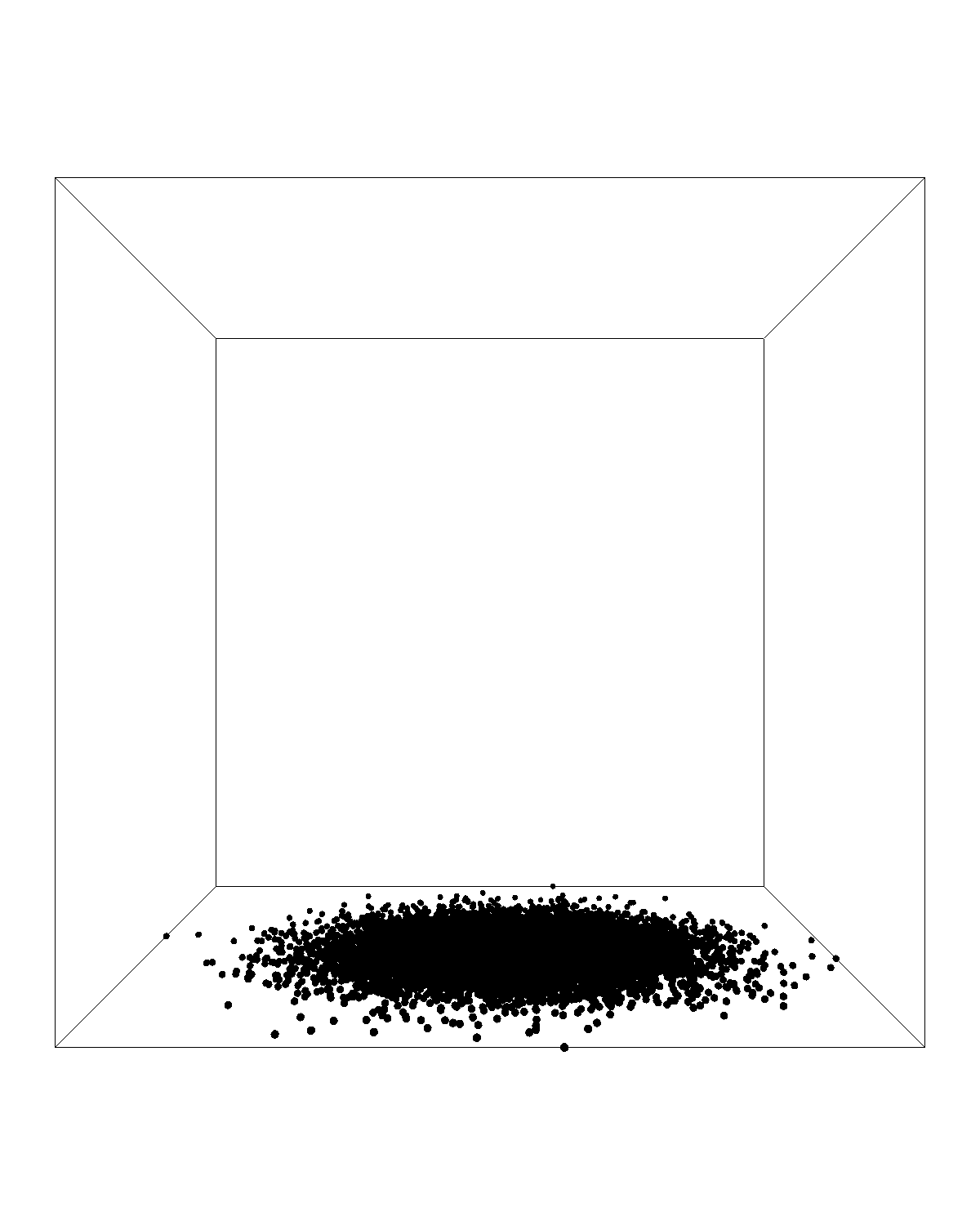}
     \hspace{0.5cm}
     \includegraphics[width=2cm]{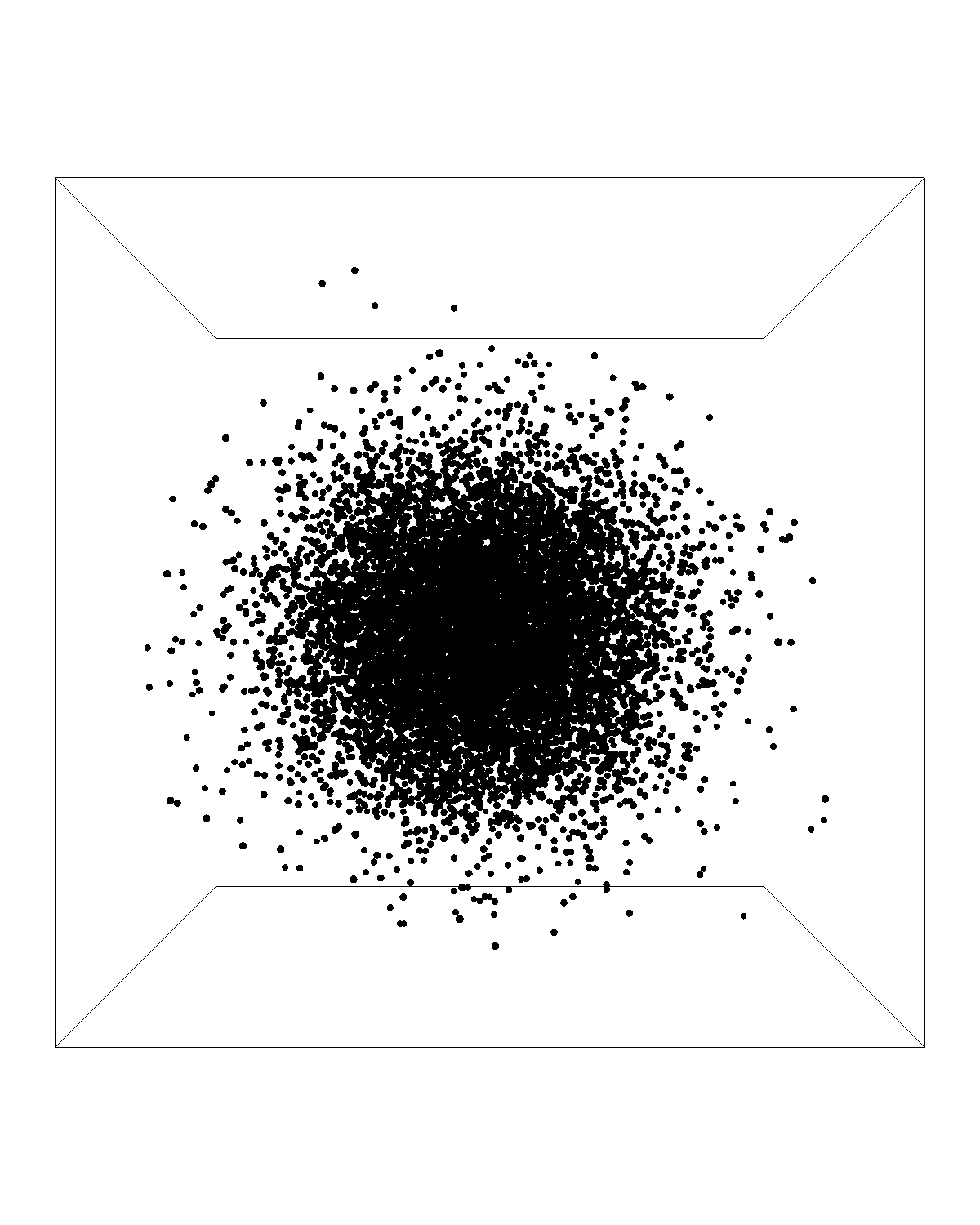}
     \caption{From left to right, a line, disk, and ball embedded in 3D space.     }
     \label{fig:dim_used_pics}
 \end{figure}

A distribution is  \textit{isotropic} when variance is uniformly distributed across all dimensions. Namely, a distribution is fully isotropic when the covariance matrix is proportional to the identity matrix. Several authors suggest that isotropy correlates with improved performance of embedding models \citep{naacl_isotropy, Wang2019ImprovingNL, Coenen2019VisualizingAM,Gong2018FRAGEFW,Hasan2017WordRV, Hewitt2019ASP, Liang, Zhou2019GettingIS,Zhou2021IsoBNFB}. However, current methods of measuring the spatial utilization of contextualized embedding models do not truly measure isotropy. The most commonly used methods for measuring spatial distribution in embedding spaces include average random cosine similarity, the partition score, variance explained and intrinsic dimensionality estimation. 
In Section \ref{sec:experiments} we argue that all current methods of measuring isotropy have fundamental shortcomings that render them inadequate measures of spatial distribution.

To overcome these limitations, we introduce \textit{IsoScore}: a novel tool for measuring the extent to which the variance of a point cloud is uniformly distributed across all dimensions in vector space. 
In contrast to previous attempts of measuring isotropy, IsoScore is the first score that incorporates the \textit{mathematical definition} of isotropy into its formulation. As a result, IsoScore has the following desirable properties that surpass the capabilities of existing metrics: (i) It is a global measure of how uniformly distributed points are in vector space that is robust to changes in the distribution mean and scalar changes in covariance; (ii) It is rotation invariant; (iii) It increases linearly as more dimensions are utilized; and (iv) It is not skewed by highly isotropic subspaces within the data. This paper makes the following novel contributions.
\begin{enumerate}
    \item This paper outlines essential conditions for measuring isotropy and uses a testing suite to empirically verify if a given method meets these conditions.
    \item We highlight fundamental shortcomings of state-of-the-art tools and demonstrate that none of the existing methods accurately measure isotropy.
    \item We present IsoScore, the first rigorously defined method for measuring isotropy in point clouds of data. 
    \item We share an efficient Python implementation of IsoScore with the community.\footnote{\href{https://github.com/bcbi-edu/p\_eickhoff\_isoscore}{\textit{https://github.com/bcbi-edu/p\_eickhoff\_isoscore}}. Alternatively:\ \texttt{pip install IsoScore}.} 
\end{enumerate}

The remainder of this paper is structured as follows: Section~\ref{sec:related} reviews previous works attempting to study isotropy in contextualized word embeddings. Section~\ref{sec:embedding_space_utilization} formally defines isotropy and describes existing tools in detail. The formal definition of IsoScore is presented in Section~\ref{sec:formal_defn} and in  Section~\ref{sec:experiments}, we report empirical results from experiments on contextualized word embeddings. Finally, Section~\ref{sec:conclusion} concludes with an outlook on future directions of work.

\section{Related Work} \label{sec:related}
\subsection{Word Embeddings}
In recent years, there has been an increased interest in analyzing the spatial organization of point clouds induced by word embeddings \citep{naacl_isotropy, bert_mean, bert_context, coenen_bert_geometry, cai_isotropy_context, vec_postprocess_mu, Liang}. Several studies have concluded that contextualized embeddings form highly anisotropic, ``narrow cones'' in vector space \citep{bert_context,cai_isotropy_context, gao2019representation, Gong2018FRAGEFW}. The most prevalent tools used to quantify the geometry of word embedding models calculate the average cosine similarity of a small number of randomly sampled pairs of points in embedding space. \citet{bert_context} claims that in some cases, contextualized embedding models have an average random cosine similarity that approaches $1.0$, meaning all points are oriented in the same direction in space irrespective of their syntactic or semantic function. 

\begin{figure}[h!]  
    \centering
    \includegraphics[width=3.7cm]{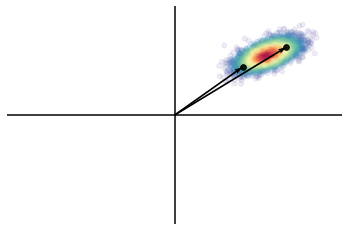}
    \includegraphics[width=3.7cm]{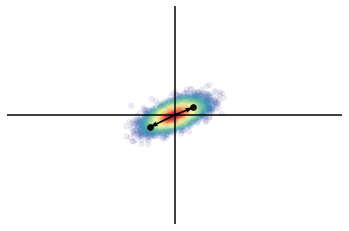}
    \caption{\emph{Left:} Point cloud $X \subset \mathbb{R}^{2}$. \emph{Right:} Result of applying a zero-mean transform to $X$.}
    \label{zero_mean_ex}
\end{figure}

In Section~\ref{sec:experiments}, we demonstrate that both average random cosine similarity and the partition score are significantly influenced by the mean of the data irrespective of how data points are distributed in vector space. Namely, if we normalize data to have zero-mean, average random cosine similarity and the partition score will artificially produce a score that reflects maximal isotropy. 
Figure \ref{zero_mean_ex} demonstrates that a applying a zero-mean transform to a point cloud increases the angle of randomly sampled points. Accordingly, the average random cosine of the left point cloud in Figure \ref{zero_mean_ex} approaches $1$ while the average random cosine similarity of the right point cloud approaches $0$. It is well known that word embedding models have non-zero mean vectors \citep{refine_ml_chu, Liang}. In the case of GPT-2 embeddings obtained from the WikiText-2 corpus \citep{wiki2}, we find values in the mean vector range from $-32.36$ to $198.19$. Although cosine similarity has long been used to capture the ``semantic'' differences between words in static embeddings, adapting any cosine similarity-based methods to measure isotropy obscures the true distribution of contextualized word embeddings.

\begin{table*}[h!]
\begin{center}
\resizebox{12.2cm}{!}{
\begin{tabular}{l|ccccc}
 \hline
 \hline 
 Test &  \textbf{IsoScore} &  \textbf{AvgRandCosSim} & \textbf{Partition} & \textbf{ID Score}  & \textbf{VarEx} \\ 
 \hline 
 \text{1. Mean Agnostic} & \green{\text{\cmark}} & \red{\text{\xmark}} & \red{\text{\xmark}} & \green{\text{\cmark}} & \green{\text{\cmark}}  \\
 \text{2. Scalar Covariance} & \green{\text{\cmark}} & \red{\text{\xmark}}  & \red{\text{\xmark}} & \green{\text{\cmark}} & \green{\text{\cmark}}  \\
\text{3. Maximum Variance} & \green{\text{\cmark}} & \red{\text{\xmark}} & \green{\text{\cmark}} & \red{\text{\xmark}} & \red{\text{\xmark}}\\
\text{4. Rotation Invariance} & \green{\text{\cmark}} & \green{\text{\cmark}}  & \red{\text{\xmark}}  & \green{\text{\cmark}} & \green{\text{\cmark}} \\
\text{5. Dimensions Used} & \green{\text{\cmark}} & \red{\text{\xmark}} & \red{\text{\xmark}}  & \red{\text{\xmark}} & \red{\text{\xmark}} \\
\text{6. Global Stability} & \green{\text{\cmark}} & \red{\text{\xmark}} & \green{\text{\cmark}}  & \green{\text{\cmark}} &\red{\text{\xmark}}\\
 \hline
 \hline 
\end{tabular}}
\caption{Performance of current methods for measuring spatial utilization.}
\label{table:test_scores}
\end{center}
\end{table*}

\subsection{Existing Methods}\label{subsec:existing_metrics}

We briefly review the most commonly used tools to measure the spatial distribution of point clouds $X\subseteq\mathbb R^n$.
A mathematical exposition of these tools can be found in Appendix \ref{app:existing_metrics}.

\textbf{Average Random Cosine Similarity:} We define the \textit{Average Random Cosine Similarity Score} as $1$ minus the average cosine similarity of $N=100,000$ randomly sampled pairs of points from $X$. \textbf{Note:} for ease of comparison to other methods, we calculate 1 minus the absolute value of the average random cosine similarity so that 0 would indicate minimal isotropy and 1 would indicate maximal isotropy. We demonstrate in Section ~\ref{sec:experiments} that average random cosine similarity is not a measure of isotropy.

\noindent \textbf{Partition Isotropy Score:} \citet{vec_postprocess_mu} define this score to be a particular quotient involving the partition function first proposed by \citet{arora}: $Z(c) := \sum_{x \in X} \text{exp}(c^{\text{T}} x)$, where $c$ is carefully chosen from the eigenspectrum of $XX^{\text{T}}$. 
It is believed that a score closer to $0$ indicates an anisotropic space, while a score near $1$ indicates an isotropic space. We refer to this as the \textit{Partition Score.}

\textbf{Intrinsic Dimensionality:} 
Algorithms for estimating intrinsic dimensionality aim to compute the true dimension of a given manifold from which we assume a point cloud has been sampled.
Intrinsic dimensionality has been used to argue word embedding models are anisotropic \citep{cai_isotropy_context}.
We use the MLE method to calculate intrinsic dimensionality \citep{mle-id}.
Dividing the intrinsic dimensionality of $X\subseteq\mathbb R^n$ by $n$ provides us with a normalized score of isotropy, which we refer to as the \textit{ID Score}. 

\textbf{Variance Explained Ratio:} The variance explained ratio, which we refer to as the \textit{VarEx Score}, measures how much total variance is explained by the first $k$ principal components of the data. 
We compute this by dividing the variance explained by the first $k$ principal components by $k/n$. 
The VarEx Score requires us to specify \textit{a priori} the number of principal components we wish to examine, which makes comparisons between vector spaces with different dimensions difficult and results in undesirable behavior, particularly when the dimension of the vector space is large.

Section~\ref{sec:experiments} demonstrates that all existing methods have fundamental shortcomings that make them unreliable measures of spatial distribution. Using any of the above existing tools to make claims about isotropy will be misleading as none of the described methods truly measure isotropy.

\section{Measuring Embedding Space Utilization} \label{sec:embedding_space_utilization}

\subsection{Definition of Isotropy}
\label{subsec:def_isotropy}
A distribution is \textit{isotropic} if its variance is uniformly distributed across all dimensions. Namely, the covariance matrix of an isotropic distribution is proportional to the identity matrix. 
Conversely, an \textit{anisotropic} distribution of data is one where the variance is dominated by a single dimension. 
For example, a line in $n$-dimensional vector space is maximally anisotropic. Robust isotropy metrics should return maximally isotropic scores for balls and minimally isotropic (i.e.\ anisotropic) scores for lines. 
Appendix~\ref{app:interpretation} provides a geometric interpretation of ``medium isotropy''.
We interpret a medium isotropic space in $\mathbb{R}^{n}$ to be one where the data uniformly utilizes approximately $n/2$ dimensions in space as defined below. Note that we exclude two edge cases for measuring isotropy. Firstly, since isotropy is a property of the covariance matrix of a distribution, the dimensionality of the space needs to be greater than 1. Secondly, we do not consider the extreme case where the data consists of a single point.

\subsection{Dimensions utilized} \label{subsec:dim_used}
Given a point cloud $X\subseteq\mathbb R^n$, we measure how many dimensions of $\mathbb R^n$ are truly utilized by $X$.
For example, we denote by $I_{n}^{(k)}$ the $n \times n$ covariance matrix where  $a_{i,i} = 1$ for $i \in \{1,2,...,k\}$ and all other elements are $0$. 
Note that when $k=n$, we recover the identity matrix. Thus, $I_{n}^{(k)}$ represents a covariance matrix where the first $k$ dimensions are being uniformly utilized.
Figure~\ref{fig:dim_used_pics} illustrates point clouds in $\mathbb R^3$ that have covariance matrix $I_3^{(1)}$, $I_3^{(2)}$, and $I_3^{(3)}$.
These utilize 1, 2, and 3 dimensions in $\mathbb R^3$.
To make this discussion rigorous and general, we make the following definition:

\begin{table*}[h!]
\begin{center}
\begin{tabular}{lllllllll}
 \hline
 \hline 
   $\iota(I_{9}^{(1)})$  &$\iota(I_{9}^{(2)})$ &$\iota(I_{9}^{(3)})$  &$\iota(I_{9}^{(4)})$  &$\iota(I_{9}^{(5)})$  &$\iota(I_{9}^{(6)})$ &$\iota(I_{9}^{(7)})$  &$\iota(I_{9}^{(8)})$ &$\iota(I_{9}^{(9)})$\\
 \hline 
 0.000 & 0.125 & 0.250 & 0.375 & 0.500 & 0.625 & 0.750 & 0.875 & 1.000\\
  \hline
 \hline 
\end{tabular}
\caption{Linearly increasing dimensions utilized in $\mathbb R^9$ linearly increases IsoScore. We prove in Appendix~\ref{app:interpretation} that IsoScore satisfies the formula $\iota(I_n^{(k)})=\frac{k-1}{n-1}$.} 
\label{table:linearly_increase_dims}
\end{center}
\end{table*}

\begin{definition}\label{def:uniform_utilization}
Consider a point cloud $X\subseteq\mathbb R^n$.
Let $\Sigma$ be the covariance matrix of $X$ and assume all the off-diagonal entries of $\Sigma$ are zero.
Let $\Sigma_D\in\mathbb R^n$ denote the diagonal of $\Sigma$.
\begin{enumerate}
    \item We say $X$ \emph{utilizes $k$ dimensions} in $\mathbb R^n$ if the  first $k$ entries of $\Sigma_D$ are non-zero and the remaining $n-k$ entries are zero.
    \item We say $X$ \emph{uniformly utilizes $k$ dimensions} in $\mathbb R^n$ if $X$ utilizes $k$ dimensions in $\mathbb R^n$ and if all the non-zero entries in $\Sigma_D$ are equal.
\end{enumerate}
\end{definition}

Having a diagonal sample covariance matrix $\Sigma$ implies there are no correlations between any coordinates of $X$. 
In Section~\ref{sec:formal_defn}, we reduce the general case of $X$ to the case where the covariance matrix of $X$ is diagonal. Figure~\ref{fig:2d_max_var} illustrates three point clouds in $\mathbb R^2$ that each utilize $2$ dimensions. We argue that it is of practical importance to differentiate between the cases in Figure~\ref{fig:2d_max_var}.
The leftmost panel uniformly utilizes all dimensions of $\mathbb R^2$, while the rightmost panel does not uniformly utilize two dimensions of space. Note that average random cosine similarity returns maximal isotropy scores for each point cloud pictured in Figure~\ref{fig:2d_max_var}.

Our proposed IsoScore reflects the dimensions utilized by a point cloud in a linear fashion.
See Table~\ref{table:linearly_increase_dims} for a concrete example of how IsoScore reflects dimensions utilized in $\mathbb R^9$.

\begin{figure}[h]
\begin{center} 
    \includegraphics[width=7.4cm]{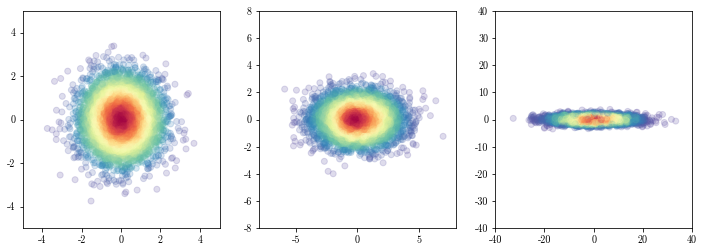}
\end{center}
    \caption{Points sampled from a 0 mean, 2D Gaussian with covariance $\begin{psmallmatrix} x & 0 \\ 0 & 1 \end{psmallmatrix}$ where $x = 1, 3, 75$.}
    \label{fig:2d_max_var} 
    \end{figure}

\subsection{Essential Properties of Isotropy}\label{sec:axioms}
\vspace{-0.05cm}
We now outline the essential properties that a measure of isotropy must possess. 

\textbf{1: Mean Agnostic.} Recall that a distribution is isotropic if variance is uniform across all dimensions. It is essential to note that \textit{isotropy is strictly a property of the covariance matrix} of a distribution. If changes to the mean of a distribution influence an isotropy score, then the given score does not measure isotropy.

\textbf{2: Scalar Changes to the Covariance Matrix.} Since isotropy is defined as the \textit{uniformity} of variance across all dimensions, isotropy scores should not change when we multiply the covariance matrix of the underlying distribution of the data by a positive scalar value. If the covariance matrix of a distribution of data is equal to $\lambda \cdot I_{n}$ where $\lambda > 0$ is some scalar value and $I_{n}$ is the $n \times n$ identity matrix, then a tool must return an isotropy score approaching $1$.

\textbf{3: Maximum Variance.} 
As we increase the difference between the maximum variance value in our covariance matrix and the average variance value of the remaining dimensions, isotropy scores should monotonically decrease to zero. 
 Figure~\ref{fig:2d_max_var} illustrates the effect of increasing the difference between the average variance value and the maximum value in the covariance matrix. 
 Increasing the difference between the maximum variance value and the average variance value increases the amount of variance explained by the first principal component of the data.
 Namely, larger maximum variance values reduce the efficiency of spatial utilization.

\textbf{4: Rotation Invariance.} 
Given a point cloud $X \subset \mathbb{R}^{n}$, an ideal measure of spatial utilization should remain constant under rotations of $X$ since the distribution of principal components remains constant under rotation. Accordingly, we consider the canonical distribution of the variance of $X$ to be the variance after projecting $X$ using principal component analysis.
Figure~\ref{fig:gauss_rotations} illustrates the process of PCA-reorientation.

\textbf{5: Dimensions Used.} 
 As described in Subsection~\ref{subsec:dim_used}, there is a direct link between isotropy and the number of dimensions utilized by the data. 
 Intuitively, increasing the number of dimensions uniformly utilized by the data expands the number of principal components it takes to explain all of the variance in the data. 
 Accordingly, a good score of spatial utilization should increase linearly as we increase the number of dimensions uniformly utilized by the data. 
 Figure~\ref{fig:dim_used_pics} depicts data utilizing one, two, and three out of three ambient dimensions, respectively.
 
   \begin{figure}[h!]  
    \centering
    \includegraphics[width=3cm]{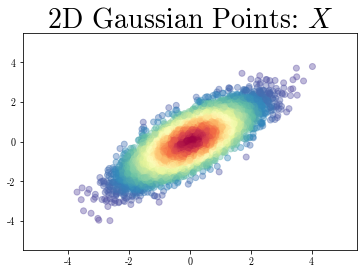}
    \hspace{0.1cm}
    \includegraphics[width=3cm]{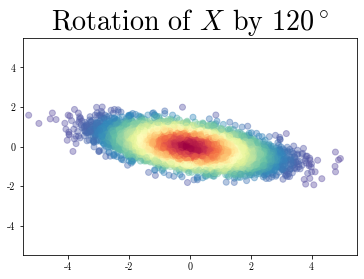}
     \hspace{0.1cm}
    \includegraphics[width=3cm]{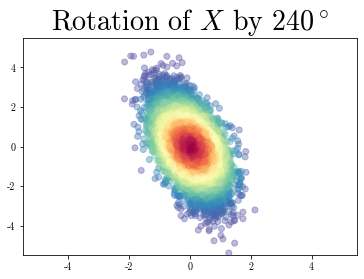}
     \hspace{0.1cm}
    \includegraphics[width=3cm]{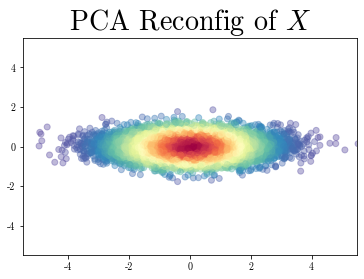}
    \caption{Left: 2D zero-mean Gaussian with covariance $\begin{psmallmatrix}
    1 & 0.8 \\ 0.8 & 1
    \end{psmallmatrix}$. We rotate $X$ by $120^{\circ}$ and $240^{\circ}$, respectively. Right: Points after applying PCA reorientation.}
    \label{fig:gauss_rotations}
    \end{figure}

\vspace{-0.1cm}
 
\textbf{6: Global stability.}
  A metric of efficient spatial utilization should be a \textit{global} reflection of the distribution. 
  A robust method should be stable even when the data exhibits small subpopulations where a score would return an extreme value. 
  
\begin{algorithm*}[h]
\caption{IsoScore}\label{algo:isoscore}
\begin{algorithmic}[1]
\Begin \ Let $X \subset \mathbb{R}^{n}$ be a finite collection of points. 
\State Let $X^{\mathrm{PCA}}$ denote the points in $X$ transformed by the first $n$ principal components.
\State Define $\Sigma_{D}\in\mathbb R^n$ as the diagonal of the covariance matrix of $X^{\mathrm{PCA}}$. 
\State Normalize diagonal to 
$\hat\Sigma_D:=\sqrt{n}\cdot\Sigma_D/\|\Sigma_D\|$, where $\|\cdot\|$ is the standard Euclidean norm.
\State The isotropy defect is $\delta(X):=\|\hat\Sigma_{D} - \mathbf 1\|/\sqrt{2(n - \sqrt{n})}$, where $\mathbf 1=(1,\dots,1)^\top\in\mathbb R^n$.
\State $X$ uniformly occupies $\phi(X):=(n-\delta(X)^2(n-\sqrt n))^2/n^2$ percent of ambient dimensions.
\State Transform $\phi(X)$ so it can take values in $[0,1]$, via $\iota(X):=(n\cdot\phi(X)-1)/(n-1)$.
\State \textbf{return:} $\iota(X)$
\End 
\end{algorithmic}
\end{algorithm*}

  We test this by computing IsoScore for the union of a noisy sphere and a line; we provide a geometric rendering of this in Figure \ref{fig:meatball} in Appendix \ref{appendix:num_experiments}. We refer to this test as the ``skewered meatball'' test. A good score of spatial distribution for a ``skewered meatball'' should reflect the ratio of the number of points sampled from the line and the number of points sampled from the sphere.
  
 \begin{figure}[h!]
        \centering
        \includegraphics[width=2.5cm]{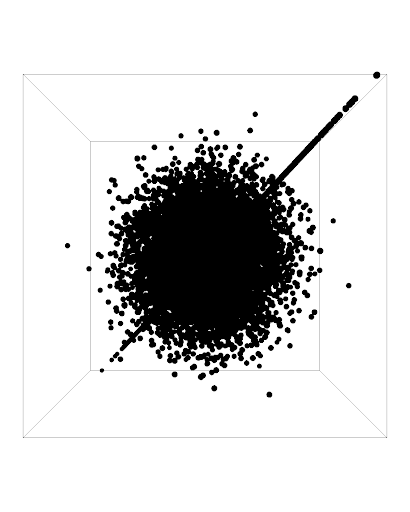}
        \caption{2D rendering of a line in 3D space intersecting  noisy sphere. AKA ``skewered meatball.''}
        \label{fig:meatball}
\end{figure}
  
 In Table~\ref{table:test_scores}, we list which existing methods satisfy which essential conditions. Section~\ref{sec:experiments} outlines the numerical experiments we execute to obtain this table. As each of the above properties have been derived from the mathematical definition of isotropy, an accurate tool for measuring isotropy needs to satisfy each essential condition.
 
 \vspace{-0.075cm}

\section{IsoScore}\label{sec:formal_defn}
\vspace{-0.05cm}

This section introduces the proposed IsoScore metric of uniform spatial utilization.

\subsection{Formal Definition of IsoScore}\label{subsec:isoscore_defn}

 Algorithm~\ref{algo:isoscore} gives a high-level overview of the procedure. 
Afterwards, we discuss the individual steps in detail.

\textbf{Step 1: Start with a point cloud $X\subseteq\mathbb R^n$}. IsoScore takes as input a finite subset of $\mathbb R^n$ and outputs a number in the interval $[0,1]$ that represents the extent to which $X$ is isotropic.

\textbf{Step 2: PCA-reorientation of data set.} Execute PCA on $X$, where the target dimension remains the original $n$. Performing PCA reorients the axes of $X$ so that the $i$'th coordinate accounts for the $i$'th greatest variance. Further, it eliminates all correlation between dimensions making the covariance matrix diagonal. We denote the transformed space as $X^{\mathrm{PCA}}$.

\textbf{Step 3: Compute variance vector of reoriented data.} Compute the $n\times n$ covariance matrix of $X^{\mathrm{PCA}}$; denote this matrix by $\Sigma$. Let  $\Sigma_D$ denote the diagonal of the covariance matrix.
    We refer to  $\Sigma_D$ as the \emph{variance vector,} and we identify $\Sigma_D$ as a vector in $\mathbb R^{n}$. 
    Performing Step 2 causes all off-diagonal entries of the covariance matrix of $X_T$ to vanish, which allows us to ignore off-diagonal elements for the rest of the computation.

\textbf{Step 4: Length normalization of variance vector.} We define the \emph{normalized variance vector} to be
    \[\hat\Sigma_D:=\sqrt{n}\cdot \frac{\Sigma_D}{\|\Sigma_D\|},\]
    where $\|(x_1,...,x_n)\|:=\sqrt{x_1^2+\cdots+x_n^2}$ denotes the standard Euclidean norm on $\mathbb R^n$.
    Note that as a result of this normalization, we have $\|\hat\Sigma_D\|=\sqrt n$.
    
\textbf{Step 5: Compute the distance between the covariance matrix and identity matrix.}
    Denote the diagonal of the $n\times n$ identity matrix by $\mathbf 1\in\mathbb R^{n}$.
    Then we define the \emph{isotropy defect} of $X$ to be
    \[\delta(X)
        :=\frac{\|\hat\Sigma_D-\mathbf 1\|}{\sqrt{2(n-\sqrt n)}}.\]
    By definition of the Euclidean norm, we have  $\|\hat\Sigma_D\|=\|\mathbf 1\|=\sqrt{n}$. It follows from the triangle inequality that $\|\hat\Sigma_D-\mathbf 1\|\in[0,2\sqrt n]$.
    Crucially, we prove in Appendix~\ref{appendix:Bounds} that achieving a value of $2\sqrt n$ using a valid covariance matrix is impossible. The largest value that can be attained is with the matrix $(a_{ij})_{i,j=1,\dots,n}$ defined by $a_{11}=\sqrt n$ and $a_{ii}=0$ whenever $i>1$.
    One can compute that the Euclidean norm in this case is $\|\hat\Sigma_D-\mathbf 1\|=\sqrt{2(n-\sqrt n)}$. Choosing this normalization factor guarantees that $\delta(X)\in [0,1]$, where $0$ represents a perfectly isotropic space and $1$ represents a perfectly anisotropic space.
    
\textbf{Step 6: Use the isotropy defect to compute percentage of dimensions isotropically utilized.} We argue in Heuristic~\ref{heuristic} that if $X$ has isotropy defect $\delta(X)$, then $X$ isotropically occupies approximately $k(X)=(n-\delta(X)^2(n-\sqrt n))^2/n$ dimensions in $\mathbb R^n$.
    Because $\delta(X)\in[0,1]$, one can estimate that $k(X)\in[1,n]$ so the fraction of dimensions utilized is $\phi(X):=k(X)/n\in[1/n,1]$.

\begin{figure*}[h!]
\centering
\includegraphics[width=13cm]{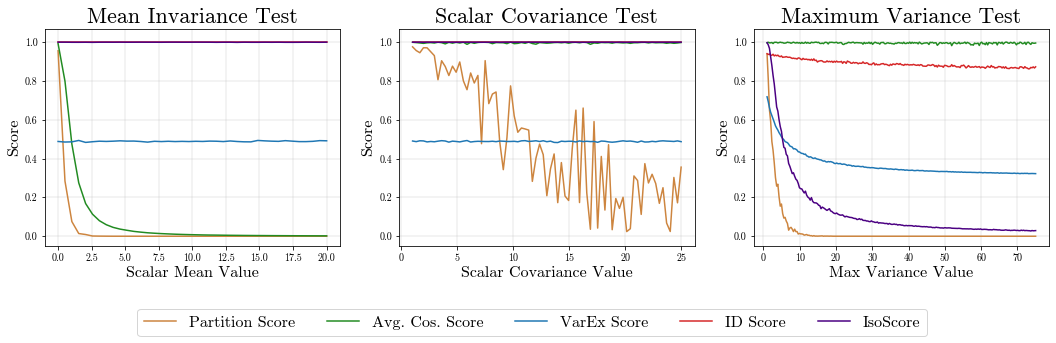}
\caption{Left: Scores of points sampled from a 10-dimensional Gaussian with identity covariance and common mean vector ranging from 0 to 20. Center: Scores for the scalar covariance test for a 5-dimensional, zero-mean Gaussian. Right: Scores for the Maximum Variance test for 10-dimensional, zero-mean Gaussians.}
    \label{fig:first_3_tests}
\end{figure*}

\textbf{Step 7: Linearly scale percentage of dimensions utilized to obtain IsoScore.}
    The fraction of dimensions utilized, $\phi(X)$, is close to the final IsoScore, but it falls within the interval $[1/n,1]$. 
    As we want the possible range of scores to fill the interval $[0,1]$, we apply the affine function that maps $1/n\mapsto 0$ and $1\mapsto 1$. Thus, $S:[1/n,1]\to[0,1]:x\mapsto(nx-1)/(n-1)$.
    Once we compose these transformations, we obtain IsoScore:
    \begin{equation}\label{eq:IsoScore_formula}
    \iota(X):=\frac{(n-\delta(X)^2(n-\sqrt n))^2-n}{n(n-1)}.
    \end{equation}

\subsection{Geometric Interpretation for IsoScore}\label{subsec:interpretation}

In Subsection~\ref{subsec:isoscore_defn} we described how to compute an IsoScore $\iota(X)$ for any point cloud $X\subseteq\mathbb R^n$.
We will now present a heuristic interpretation for a given IsoScore.
Intuitively, our heuristic says that $\iota(X)$ is roughly the fraction of dimensions of $\mathbb R^n$ utilized by $X$.
More precisely, the quantity of dimensions of $\mathbb R^n$ utilized by $X$ is some number inside the interval $[\iota(X)n,\iota(X)n+1]\cap[1,n]$.
We formalize this below.

\begin{heuristic}\label{heuristic2}
When the ambient space $\mathbb R^n$ has large dimension, the IsoScore $\iota(X)$ is approximately the fraction of dimensions uniformly utilized by $X$.
\end{heuristic}

We prove this heuristic in Appendix~\ref{app:interpretation}.
Note in particular that $\iota(X)=0$ implies that~\ref{eq:IsoScore_to_dimensions_used} simplifies to a single dimension utilized and $\iota(X)=1$ implies that~\ref{eq:IsoScore_to_dimensions_used} simplifies to all $n$ dimensions utilized.

Because IsoScore covers a continuous spectrum, one should carefully interpret what we mean when we say that $X$ occupies approximately $k$ dimensions of $\mathbb R^n$.
For example, consider the 2D Gaussian distributions depicted in Figure~\ref{fig:2d_max_var}. 
Heuristic~\ref{heuristic} predicts $k=1.9996,1.6105,1.0281$ dimensions are used when $x=1,3,75$, respectively.
These should be interpreted as follows: ``when $x=75$, the points sampled are mostly using one direction of space'' and ``when $x=3$, the points sampled are using somewhere between one and two dimensions of space.''

\section{Experiments} \label{sec:experiments}
In Subsection~\ref{subsec:testing_axioms}, we present results from numerical experiments designed to test each of the isotropy scores presented in this paper against the six essential properties outlined in Section~\ref{sec:axioms}. Exact descriptions of the numerical experiments are provided in Appendix~\ref{appendix:num_experiments}. We reiterate that each of the essential conditions have been derived directly from the mathematical definition of isotropy and violating any of the essential properties disqualifies a method from being a correct measure of isotropy. 

In Subsection~\ref{subsec:context_embeddings}, we demonstrate the merit of IsoScore by recreating the experimental setup presented in \citep{cai_isotropy_context}. We create word embeddings for tokens from the WikiText-2 corpus using GPT \citep{gpt}, GPT-2 \citep{gpt2}, BERT \citep{bert} and DistilBERT \citep{distil-bert} and calculate isotropy scores for each layer of the model. 

\subsection{Testing methods against the essential properties}\label{subsec:testing_axioms}

\textbf{Test 1: Mean Agnostic.} 
 When the covariance matrix of a distribution is proportional to the identity matrix, measures of isotropy should return a score of $1$ regardless of the value of the mean. Figure~\ref{fig:first_3_tests} demonstrates that neither average random cosine similarity nor the partition score are mean-agnostic.
IsoScore is mean-agnostic since it is a function of the covariance matrix. 
Importantly average random cosine similarity and the partition score are skewed by non-zero mean data. Our results show that, for an isotropic Gaussian with covariance matrix $\lambda \cdot I_{n}$ and mean vector $M = [\mu, \mu, ..., \mu]$, the average random cosine similarity of points sampled from this distribution will approach 0 as we increase the ratio between $\mu/\lambda$. \textit{Consequently, zero-centering data will cause average random cosine similarity to return maximally isotropic scores without impacting the distribution of the variance.} 

\textbf{Test 2: Scalar Changes to the Covariance Matrix.}
For a 5-dimensional Gaussian distribution with a zero mean vector and covariance matrix $\lambda \cdot I_{n}$, scores should reflect uniform utilization of space for any $\lambda > 0$. Figure~\ref{fig:first_3_tests} shows that IsoScore and the intrinsic dimensionality score are the only metrics that are agnostic to scalar multiplication to the covariance matrix and return a score 1 for each value of $\lambda$. In Step 4 of IsoScore, we normalize the diagonal of the covariance matrix to have the same norm as the diagonal of the identity matrix, which ensures IsoScore is invariant to scalar changes in covariance.

\begin{table}[h]
\begin{center}
\caption{Performance of current methods on Test 4: Rotation Invariance}
\label{table:rotation_invariance}
\resizebox{\columnwidth}{1cm}{
\begin{tabular}{l|ccccc}
 \hline
 \hline 
   &  \textbf{\emph{IsoScore}} & \textbf{AvgCosSim}  &  \textbf{Partition}   & \textbf{ID Score} & \textbf{VarEx}  \\
 \hline 
    $X$   & 0.216 & 0.990 & 0.445 & 1.000 & 0.500 \\
    $X^{120^\circ}$ & 0.216  & 0.968 & 0.673  & 1.000 & 0.500\\
    $X^{240^\circ}$ & 0.216  & 0.981 &  0.669  & 1.000 & 0.500\\
    $X^{\text{PCA}}$  & 0.216  & 0.993 & 0.446   & 1.000 & 0.500\\
\hline
 \hline 
\end{tabular}}
\end{center}
\end{table}

\begin{figure*}[h!]
    \centering
    \includegraphics[width=13cm]{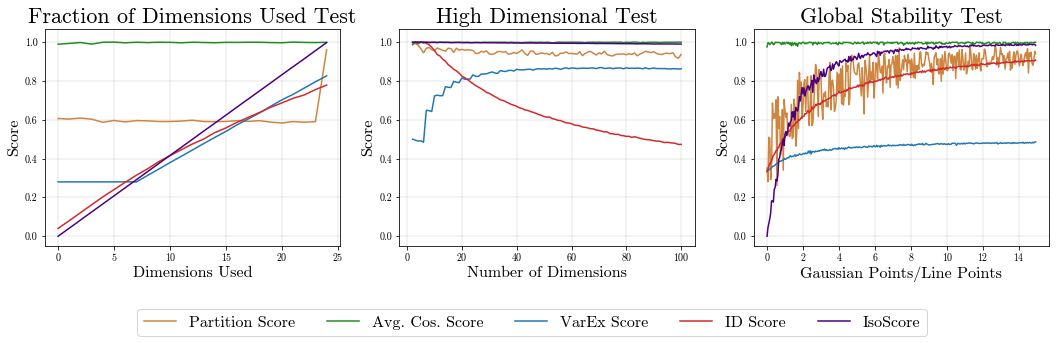}
    \caption{Left and center: Scores for the two Dimensions Used tests. Right: Scores for the ``skewered meatball'' test in 3 dimensions.}
    \label{fig:num_exp2}
\end{figure*}

\textbf{Test 3: Maximum Variance.} An effective score should monotonically decrease to $0$ as we increase the difference between the maximum variance value and average variance. Steps 4 and 5 of IsoScore ensure that the less equitably the mass in the covariance vector is distributed, the greater the isotropy defect will be. Figure~\ref{fig:2d_max_var} visualizes this phenomenon for a 2 Dimensional Gaussian. The ID Score fails this test since the intrinsic dimensionality estimate is 2.0 for all point clouds depicted in Figure~\ref{fig:2d_max_var}.

\textbf{Test 4: Rotation Invariance.} We rotate our baseline point cloud $X$ by $120^{\circ}$ and $240^{\circ}$. Lastly, we project $X$ using PCA reorientation while retaining dimensionality to obtain a point cloud $X^{\text{PCA}}$. We record results in Table~\ref{table:rotation_invariance}.
Only IsoScore, ID Score, and VarEx Score return constant values. The partition score would return a constant value if it were feasible to compute the true optimization problem. The approximate version of the partition score, however, depends too strongly on the basis. IsoScore is rotation invariant by design. In Step 2, IsoScore projects the point cloud of data in the directions of maximum variance before computing the covariance matrix of the data. 

 \textbf{Test 5: Dimensions Used (Fraction of Dimensions Used Test).}
The number of dimensions used in a point cloud $X \subset \mathbb{R}^{n}$ provides a sense of how uniformly $X$ utilizes the ambient space. A reliable metric should return scores near $0.0, 0.5,$ and $1.0$ when number of dimensions used is $1, \lfloor n/2 \rfloor,$ and $n$, respectively. Figure \ref{fig:num_exp2} shows that only IsoScore models ideal behavior for the dimensions used test. A rigorous explanation of why  IsoScore reflects the percentage of 1s present in the diagonal of the covariance matrix is provided in Heuristic~\ref{heuristic2}. Although the intrinsic dimensionality score monotonically increases as we increase $k$, it fails to reach $1$ when all dimensions are uniformly utilized. Average cosine similarity fails this test, as it stays constant near $1$ regardless of the fraction of dimensions uniformly utilized.

\textbf{Test 5: Dimensions Used (High Dimensional Test).}
Metrics of spatial utilization should allow for easy comparison between different vector spaces even when the dimensionality of the two spaces is different. Figure~\ref{fig:num_exp2} illustrates that IsoScore, the average cosine similarity score, and the partition score pass this test, as they stay constant near $1$. Note that the line for IsoScore decreases slightly. By the law of large numbers, the more data points we sample from the Gaussian distribution, the closer the covariance matrix will be to the covariance matrix from which it was sampled. The VarEx Score is not stable under an increase in dimension primarily because it requires the user to specify the percentage of principal components used in calculating the score. Note that the ID Score begins to decrease simply by increasing the dimensionality of the space since the MLE method is not very well suited for estimating the intrinsic dimension of isotropic Gaussian balls.

\textbf{Test 6: Global Stability.}
 To evaluate which scores are not skewed by highly concentrated subspaces, we design the ``skewered meatball test'' (see Figure~\ref{fig:meatball} for a geometric rendering).
As we increase the ratio between the number of points sampled from a 3D isotropic Gaussian and a 1D anisotropic line, we should see isotropy scores increase from $0$ to $1$, and hit $0.5$ precisely when the number of points sampled from the Gaussian distribution and the line are equal. 
Results from the skewered meatball test in Figure~\ref{fig:num_exp2} indicate that the partition score, IsoScore and intrinsic dimensionality estimation are the only metrics that are global estimators of the data.

\subsection{Isotropy in Contextualized Embeddings} 
\label{subsec:context_embeddings}

Recent literature suggests that contextualized word embeddings are anisotropic. However, as demonstrated in Subsection~\ref{subsec:testing_axioms}, no existing methods truly measure isotropy. We replicate experiments by \citep{cai_isotropy_context}, and present isotropy scores for the vector space of token embeddings generated from the WikiText-2 corpus for GPT (110M parameters) and GPT2 (117M parameters) in Figure~\ref{fig:GPT_scores}, as well as the scores for BERT (base, uncased) and DistilBERT (base, uncased) in Figure~\ref{fig:BERT_scores}.

\begin{figure}[h!]
\begin{center}
    \resizebox{\columnwidth}{!}{\includegraphics[width=10.5cm]{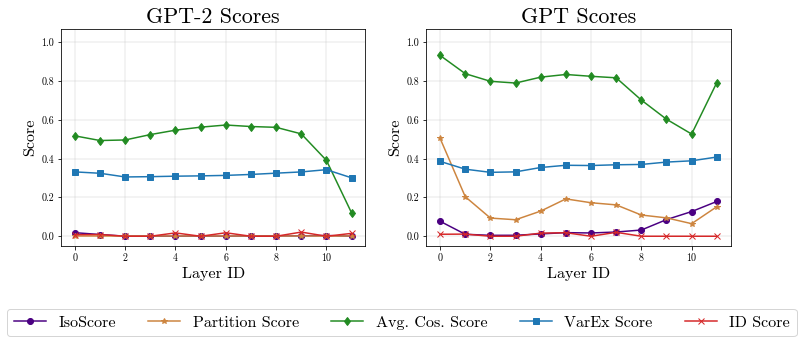}}
\caption{The 5 scores for each of the 12 layers of GPT-2 and GPT}
\label{fig:GPT_scores}
\end{center}
\end{figure}

\begin{figure}[h!]
\begin{center}
    \resizebox{\columnwidth}{!}{\includegraphics[width=10.5cm]{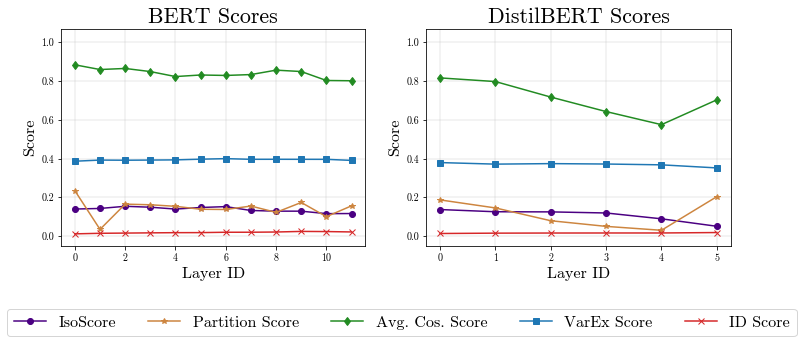}}
\caption{The 5 scores for the 12 layers of BERT, and the 6 layers of DistilBERT}
\label{fig:BERT_scores}
\end{center}
\end{figure}

Our findings using IsoScore challenge and extend upon the literature in the following ways. Contextualized embedding models  (i) utilize even fewer dimensions than previously thought; (ii) do not utilize fewer dimensions in deeper layers; and (iii) in agreement with \citet{naacl_isotropy}, contextualized embedding models do not necessarily occupy a ``narrow cone'' in space. 

IsoScore returns values of less than 0.18 for every considered contextualized embedding model. GPT and GPT-2 embeddings do not even isotropically utilize a single dimension in space, in the sense of Heuristic~\ref{heuristic}. Using average random cosine similarity, Cai et al.\ concluded that earlier layers in contextualized embedding models are more isotropic than layers deeper in the network. While this may appear to be true using inaccurate measures of isotropy, there is no significant decrease in IsoScore between the earlier and later layers of contextualized embedding models. \citet{naacl_isotropy} argue that isotropy improves performance for contextualized embedding models and that enforcing zero mean embeddings recovers ``isotropy''. The author's claim to improve isotropy by subtracting the mean vector from the point clouds of embeddings produced from BERT, GPT-2 and RoBERTa, however, the authors use the partition score in attempts to measure isotropy which will return values close to 1 when the data is zero-mean. As demonstrated throughout the paper, isotropy is strictly a property of the covariance matrix of a distribution and is by definition mean-agnostic. 

Note that our average random cosine similarity score finds contextualized embedding models to be much more isotropic then previously reported. When computing the average random cosine similarity score for contextualized word embeddings we sample 250,000 pairs of points. Prior studies such as \citet{bert_context} and \citet{cai_isotropy_context} sample as few as 1000 pairs of points when calculating average random cosine similarity. In both cases, the point clouds contain millions of tokens embedded into 768 dimensional vector space and differences in reported scores are likely due to sampling noise. We found empirically that the quantity of points sampled should be orders of magnitude larger than the dimension. 

The notion of isotropy is often conflated with geometry. The geometry of isotropic vector spaces, however, will differ depending on the distribution that generates the points in space. For example, multivariate isotropic Gaussians form $n$-dimensional balls and uniform distributions form $n$-dimensional cubes, yet both distributions receive an IsoScore of $1$. 
For an illustrated example of points generated from different isotropic distributions, consult Appendix~\ref{appendix:geometry_of_isotropy}. It is therefore not necessarily the case that even highly anisotropic embedding spaces form narrow, anisotropic cones.

\section{Conclusion \& Future Works}\label{sec:conclusion}
Several studies have attempted to study isotropy in contextualized embedding models. Using mathematically rigorous tests, we demonstrate that current methods do not accurately measure isotropy. This paper
presents IsoScore: a novel method for measuring isotropy that corrects the current misunderstandings in the literature. IsoScore is the only tool that is mean agnostic, robust to scalar changes to the covariance matrix and rotation invariant. Furthermore, IsoScore scales linearly with the number dimensions used and is stable when distributions contain highly isotropic subspaces. 
Future studies should avoid using existing methods to measure isotropy as resulting conclusions will be misleading or altogether inaccurate. 

There are several promising directions for future work. Current studies have used inaccurate methods to claim that increasing isotropy improves the performance of contextualized embedding models. However, we believe that further decreasing isotropy could improve performance, especially in language modeling applications. IsoScore could be used as a regularizer when fine tuning word embeddings to penalize distributions that exhibit isotropy. 

As point clouds of data arise in nearly all deep learning applications, IsoScore presents itself
as a useful tool to study and refine a variety of models beyond the domain of NLP.

\section*{Acknowledgements}

This research is supported in part by the NSF (IIS-1956221) and NIH (T32 GM128596). The views and conclusions contained herein are those of the authors and should not be interpreted as necessarily representing the official policies, either expressed or implied, of NSF, NIH, or the U.S.\ Government.

\bibliography{anthology,custom}
\bibliographystyle{acl_natbib}

\appendix

\section{Interpretation: IsoScore as a Summary Statistic}
\label{sec:interpretation}

We will now provide an intuitive interpretation for the IsoScore of a point cloud $X\subseteq\mathbb R^n$.
The interested reader should consult Appendix \ref{app:interpretation} for an in-depth explanation of this heuristic. 

\begin{heuristic}\label{heuristic_body}
The IsoScore of $X$ is roughly the fraction of dimensions uniformly utilized by $X$.
\end{heuristic}
For example, an IsoScore near $0.5$ indicates that around half of the dimensions are utilized; and more generally, an IsoScore near $\alpha\in[0,1]$ indicates that approximately $n \cdot \alpha$ of the dimensions of $\mathbb R^n$ are uniformly utilized by $X$.
Table \ref{table:linearly_increase_dims} illustrates this trend where IsoScore increases linearly as more dimensions are uniformly utilized in $\mathbb R^9$. 

\section{Pre-existing metrics, in detail}
\label{app:existing_metrics}

\textbf{Average Cosine Similarity:} We define the \textit{Average Cosine Similarity Score} as $1$ minus the average cosine similarity of $N$ randomly sampled pairs of points from $X$. That is,
\begin{equation}\label{eq:avg_cos_sim}
\text{AvgCosSim}(X) := 1-\Bigg|\sum_{i=1}^N \frac{\text{cos}(x_{i},y_{i})}{N}\Bigg|,
\end{equation}
where $\{(x_1,y_1),\dots,(x_N,y_N)\}\subseteq X\times X$ are randomly chosen with $x_i\ne y_i$ for all $i$, and cos$(x_{i}, y_{i})$ denotes the cosine similarity of $x_{i}$ and $y_{i}$. Some authors define the average cosine similarity score to be exactly the average, rather than one minus the average. However, for ease of comparison to other metrics, our score ensures that $\text{AvgCosSim}(X)$ is between $0$ and $1$. Under our convention, it is commonly believed that a score of $0$ indicates that the point cloud $X$ is anisotropic and a score of $1$ indicates that $X$ is isotropic. In Section~\ref{sec:experiments}, we demonstrate that this is not the case.

\noindent \textbf{Partition Isotropy Score:} For any unit vector $c\in\mathbb R^n$, let the partition function be denoted as $Z(c) := \sum_{x \in X} \text{exp}(c^{\text{T}} x)$.
\citet{vec_postprocess_mu} measure isotropy as $I(X) := (\text{min}_{||c||=1} Z(c))/({\text{max}_{||c||=1}Z(c)})$. It is believed that a score closer to zero indicates an anisotropic space while a score closer to one indicates an isotropic space. \citet{vec_postprocess_mu} demonstrate that a score of $1$ implies that the eigenspectrum of $X$ is flat. Computing $I(X)$ explicitly is intractable since the set of unit vectors is infinite. 
Accordingly, \citet{vec_postprocess_mu} approximate $I(X)$ by 

    \begin{equation}\label{eq:partition_score}
        I(X) \approx \frac{\text{min}_{c \in C}Z(c)}{\text{max}_{c \in C}Z(c)}
    \end{equation}
    
\noindent where $C$ is the set of eigenvectors of $X^{\text{T}}X$. For the remainder of the paper we refer to (\ref{eq:partition_score}) as the \textit{Partition Score.}

\textbf{Intrinsic Dimensionality:} 
Given a point cloud $X\subseteq\mathbb R^n$, it is sometimes useful to assume that $X$ is sampled from a manifold of dimension less than $n$. For example, points in the left panel in Figure~\ref{fig:dim_used_pics} are sampled from a 1-dimensional space and points in the middle panel are sampled from a 2-dimensional space. Algorithms for intrinsic dimensionality aim to estimate the true dimension of a given manifold from which we assume a point cloud has been sampled.
Intrinsic dimensionality has been used to argue that word embedding models are anisotropic \citep{cai_isotropy_context}.  For a point cloud $X \subset \mathbb{R}^{n}$, it is commonly thought that the more isotropic $X$ is, the closer the intrinsic dimensionality of $X$ is to $n$. Dividing the intrinsic dimensionality of $X$ by $n$ provides us with a normalized score of isotropy, which we refer to as the \textit{ID Score}. We use the maximum likelihood estimation (MLE) method to calculate intrinsic dimensionality. For a detailed description of the MLE method for intrinsic dimensionality estimation please consult \citep{mle-id, id-survey}. 

\textbf{Variance Explained Ratio:} The variance explained ratio measures how much total variance is explained by the first $k$ principal components of the data. Note that when all principal components are considered, the variance explained ratio is equal to $1$. Examining the eigenspectrum of principal components is undoubtedly a useful tool in quantifying the spatial distribution of high dimensional data. However, the variance explained ratio requires us to specify \textit{a priori} the number of principal components we wish to examine. We divide the variance explained by the first $k$ principal components by $k/n$ to convert the variance explained ratio into a normalized score.

\section{Bounds on IsoScore}
\label{appendix:Bounds}

\begin{prop}
Let $X\subseteq\mathbb R^n$ be a finite set.
Then $\iota(X)\in [0,1]$.
\end{prop}
\begin{proof}
Define $\Sigma$ to be the $n\times n$ sample covariance matrix of $X^{\mathrm{PCA}}$.
Let $c>0$ be so that if we define $\hat\Sigma:=c\cdot\Sigma$, then $\|\hat\Sigma_D\|=\sqrt n$.
Let us enumerate the entries of this vector as $\hat\Sigma_D=(\mathrm{Var}(x_1),\dots,\mathrm{Var}(x_n))$.
In order to show that $\iota(X)\in[0,1]$, it is equivalent to show that $\|\hat\Sigma_D-\mathbf 1\|\in[0,\sqrt{2(n-\sqrt n)}]$, and by definition of the Euclidian norm, the latter estimate is equivalent to
\begin{equation}\label{eq:Iso_Score_bound_WTS}
    2(n-\sqrt n)\ \ge \ \sum_{i=1}^n(\mathrm{Var}(x_i)-1)^2.
\end{equation}
But the identity $\|\hat\Sigma_D\|=\sqrt n$ implies that $\sum_{i=1}^n\mathrm{Var}(x_i)^2=n$, so in fact (\ref{eq:Iso_Score_bound_WTS}) is equivalent to
\begin{align*}
 \sum_{i=1}^n \mathrm{Var}(x_i)\ge\sqrt{n}.
\end{align*}
If this inequality were flipped, then we could estimate that
\begin{align*}
    n
&=\mathrm{Var}(x_1)^2+\cdots +\mathrm{Var}(x_n)^2\\
&\le (\mathrm{Var}(x_1)+\cdots+\mathrm{Var}(x_n))^2\\
&<n,
\end{align*}
which is a contradiction.
\end{proof}

\section{Interpretation of IsoScore, in Detail}
\label{app:interpretation}

This appendix provides rigorous mathematical justification for the claims that we made in Appendix \ref{sec:interpretation} about the interpretation of IsoScore.
It is split into two parts.
In Appendix~\ref{app:frac_of_dims} we formalize, and prove, the claim that the IsoScore for a point cloud $X$ is approximately the fraction of dimensions uniformly utilized by $X$.
And in Appendix~\ref{app:k_n_dim} we argue that IsoScore is an honest indicator of uniform spatial utilization.


\subsection{IsoScore Reflects the Fraction of Dimensions Uniformly Utilized}\label{app:frac_of_dims}

In Section~\ref{sec:interpretation} we provided an interpretation for the value of the IsoScore $\iota(X)$ in Heuristic \ref{heuristic_body}.
Intuitively, our heuristic says that $\iota(X)$ is roughly the fraction of dimensions of $\mathbb R^n$ utilized by $X$.
We will now explain and justify this heuristic in detail.
We formalize our heuristic below.

\begin{heuristic}\label{heuristic}
Suppose that a point cloud $X\subseteq\mathbb R^n$ gives an IsoScore $\iota(X)$.
Then $X$ occupies approximately 
\begin{equation}\label{eq:IsoScore_to_dimensions_used}
     k(X) := 
     \iota(X)\cdot n+1-\iota(X)
\end{equation}
dimensions of $\mathbb R^n$.
\end{heuristic}

Note in particular that $\iota(X)=0$ implies that~(\ref{eq:IsoScore_to_dimensions_used}) simplifies to a single dimension utilized and $\iota(X)=1$ implies that~(\ref{eq:IsoScore_to_dimensions_used}) simplifies to all $n$ dimensions utilized.

In the remainder of this subsection, we will justify the above heuristic.
We will make reference to the notations and equations in Section \ref{sec:formal_defn}.
Fix $n\ge 1$ and $k\in\{1,\dots,n\}$, and consider the matrix $I_n^{(k)}$.
Recall that $I_n^{(k)}$ is the covariance matrix for a $k$-dimensional uncorrelated Gaussian distribution in $\mathbb R^n$.
For example, spaces sampled using the matrices $I_3^{(k)}$, for $k=1,2,3$ are rendered in Figure~\ref{fig:dim_used_pics} as a line, a circle, and a ball, respectively.
One can compute directly that the IsoScores for these three spaces are 
\[\iota(I_3^{(1)})\approx 0.0,\qquad
\iota(I_3^{(2)})\approx 0.5,\qquad
\iota(I_3^{(3)})\approx 1.0.
\]
Our main insight in this section is that it is worthwhile to apply these statistics for reverse reasoning in the following sense: suppose you have some point cloud $X\subseteq\mathbb R^3$ which satisfies $\iota(X)\approx 1/2$.
Then this IsoScore should allow you to infer that $X$ uniformly occupies approximately $2$ dimensions of $\mathbb R^3$.

In Heuristic~\ref{heuristic}, we provide the closed formula (\ref{eq:IsoScore_to_dimensions_used}) for generalizing the above reasoning to all dimensions $n$.
We will now prove this formula.

\begin{proof}[Proof of Heuristic~\ref{heuristic}]
Once we normalize $I_n^{(k)}$ so that its Euclidean norm is $\sqrt n$, we get that the first $k$ diagonal entries are $\sqrt{n/k}$. Therefore, the isotropy defect is
\begin{align}\label{eq:iso_defect_Ink}
\delta(I_n^{(k)})
&=\frac{\|\hat I_n^{(k)}-\mathbf 1\|}{\sqrt{2(n-\sqrt n)}}\\
&=\frac{\sqrt{k(1-\sqrt{n/k})^2+n-k}}{\sqrt{2(n-\sqrt n)}}\\
&=\frac{\sqrt{n-\sqrt{nk}}}{\sqrt{n-\sqrt n}}.\nonumber
\end{align}
It is natural to consider the map $k\mapsto\delta(I_n^{(k)})$.
A priori, this is a discrete function defined on $\{1,\dots,n\}$; a fortiori, this is in fact a continuous, monotonically decreasing bijection on the connected interval $[1,n]$.
Therefore, the function defined by
\[\tilde\delta_n:[1,n]\to[0,1]:k\mapsto\delta(I_n^{(k)})\] 
is invertible, and one can compute that its inverse is
\[\tilde\delta_n^{-1}:[0,1]\to[1,n]:d\mapsto\frac{(n-d^2(n-\sqrt n))^2}{n}.\]


The truth of this heuristic rests upon the validity of the following assumption, which is reasonable to use in many contexts.

\begin{assumptionUnderpinningHeuristic}
The isotropy defect corresponding to a point cloud sampled using the covariance matrix $I_n^{(k)}$ is the prototypical isotropy defect for any point cloud in $\mathbb R^n$ which uniformly utilizes $k$ dimensions.
\end{assumptionUnderpinningHeuristic}

We will now invoke this assumption.
Let $\delta(X)$ be the isotropy defect for an arbitrary point cloud $X$.
If we assume that we are in the nontrivial case where $\delta(X)>0$, then $\tilde\delta_n^{-1}(\delta(X))$ is in the interval $[1,n)$.
Because $\tilde\delta_n^{-1}$ is bijective, there exists a unique $k\in\{1,\dots,n-1\}$ with the property that $\tilde\delta_n^{-1}(\delta(X))\in[k,k+1).$
But by construction, $[k,k+1)=[\tilde\delta_n^{-1}(\delta(I_n^{(k)})),\tilde\delta_n^{-1}(\delta(I_n^{(k+1)})))$.
By monotonicity of $\tilde\delta_n^{-1}$, this implies that \[\delta(X)\in[\delta(I_n^{(k)}),\delta(I_n^{(k+1)})).\]
Therefore, by the assumption underpinning the heuristic, we can deduce that $X$ is uniformly utilizing between $k$ and $k+1$ dimensions of $\mathbb R^n$.
To be specific, we say that $X$ is uniformly utilizing $\tilde\delta_n^{-1}(\delta(X))\in[k,k+1)$ dimensions.
Recalling Section~\ref{sec:formal_defn}, we can recognize that 
    in Step 6, the formula for $k(X)$, the quantity of dimensions uniformly utilized by $X$, is precisely $k(X):=\tilde\delta_n^{-1}(\delta(X))$; 
    likewise, the formula for $\phi(X)$, the fraction of dimensions uniformly utilized by $X$, is $\phi(X):=\tilde\delta_n^{-1}(\delta(X))/n$.

Now we are in a position to verify Equation~\ref{eq:IsoScore_to_dimensions_used}, the main claim of Heuristic~\ref{heuristic}.
By the assumption underpinning the heuristic, it is sufficient to verify Equation~\ref{eq:IsoScore_to_dimensions_used} in the case of $I_n^{(k)}$, for $k=1,\dots,n$.
This is because all functions that we will utilize are monotonic bijections.
Using the notation in Steps 6 and 7 in Section~\ref{sec:formal_defn}, we can compute that
\begin{align*}
\iota(I_n^{(k)})(n-1)+1
&=S(\phi_n(I_n^{(k)}))(n-1)+1\\
&=n\cdot\phi_n(I_n^{(k)})\\
&=k(I_n^{(k)}).
\end{align*}
Using the formula $k(X)=(n-\delta(X)^2(n-\sqrt n))^2/n$, we can continue:
\begin{align*}
k(I_n^{(k)})
&=\frac{(n-\delta(I_n^{(k)})^2(n-\sqrt n))^2}{n}\\
&=\frac{(n-\frac{n-\sqrt{nk}}{n-\sqrt n}(n-\sqrt n))^2}{n}\\
&=k,
\end{align*}
where in the penultimate equality we used Equation~\ref{eq:iso_defect_Ink}.
This completes the proof.
\end{proof}

Because IsoScore covers a continuous spectrum, one should carefully interpret what we mean when we say that $X$ occupies approximately $k$ dimensions of $\mathbb R^n$.
For example, consider the 2D Gaussian distributions depicted in Figure~\ref{fig:2d_max_var}. 
Heuristic~\ref{heuristic} predicts $k=1.9996,1.6105,1.0281$ dimensions are used when $x=1,3,75$, respectively.
These should be interpreted as follows: ``when $x=75$, the points sampled are mostly using one direction of space'' and ``when $x=3$, the points sampled are using somewhere between one and two dimensions of space.''

Heuristic \ref{heuristic} suggests that an IsoScore near $1/2$ means that the corresponding point cloud $X$ occupies approximately half of the dimensions of its ambient space. 
We can make this reasoning rigorous as follows: for any $n\ge 2$, one can compute that  
\begin{equation}\label{eq:closed_formula}
 \iota(I_n^{(k)})
 =\frac{k-1}{n-1}
 \approx\frac{k}{n},\qquad\text{for any $k=1,\dots,n$.}   
\end{equation}
\begin{proof}[Proof of (\ref{eq:closed_formula})]
In Equation~\ref{eq:iso_defect_Ink} computed that the isotropy defect is $\delta(I_n^{(k)})=\sqrt{n-\sqrt{nk}}/\sqrt{n-\sqrt n}$. 
If we substitute this expression into (\ref{eq:IsoScore_formula}), then we obtain the formula $\iota(I_n^{(k)})=\frac{k-1}{n-1}.$
Furthermore, one can easily estimate that $|\frac{k-1}{n-1}-\frac{k}{n}|\le\frac{1}{n}$.
\end{proof}
Table \ref{table:linearly_increase_dims} illustrates this formula in the concrete case of $\mathbb R^9$.
This formula implies the following key relationship: \[\lim_{n\to\infty}\iota(I_n^{(\left\lfloor{n/2}\right\rfloor)})=1/2.\]
Generalizing this line of reasoning yields our second heuristic explanation for the meaning of IsoScore, Heuristic~\ref{heuristic2}.
We copy it here:

\begin{heuristic}\label{heuristic2_copy}
When the ambient space $\mathbb R^n$ has large dimension, the IsoScore $\iota(X)$ is approximately the fraction of dimensions uniformly utilized by $X$.
\end{heuristic}
\begin{proof}[Proof of Heuristic~\ref{heuristic2}]
By the assumption underpinning Heuristic~\ref{heuristic}, it suffices to show this in the case of matrices of the form $I_n^{(k)}$.
Fix $\alpha\in[0,1]$, and consider the covariance matrix $I_n^{(\lfloor \alpha n\rfloor)}$.
For large $n$, the fraction of dimensions uniformly utilized by $I_n^{(\lfloor \alpha n\rfloor)}$ is approximately $\alpha$, according to Definition~\ref{def:uniform_utilization}.
But by (\ref{eq:closed_formula}), we can compute that
\[\lim_{n\to\infty}\iota(I_n^{(\lfloor \alpha n \rfloor)})
=\lim_{n\to\infty}\frac{\lfloor \alpha n \rfloor-1}{n-1}
=\alpha.
\]
This completes the proof.
\end{proof}

\subsection{The IsoScore for \texorpdfstring{$I_n^{(k)}$}{Ink} Reflects Uniform Utilization of \texorpdfstring{$k$}{k} Dimensions}
\label{app:k_n_dim}

We will now investigate what range of IsoScores are obtained by sample covariance matrices that utilize $k$ out of $n$ dimensions.
It is easy to see that these scores at least fill the interval $(0,\iota(I_n^{(k)})]$, since the map
\begin{align*}
[1,&\infty)\to(0,\iota(I_n^{(k)})]\\
&x\mapsto\iota(\diag(x,1,\dots,1,0,\dots,0))
\end{align*}

is surjective.
Conversely, we can show that this interval is the only possible range of IsoScores corresponing to such covariance matrices.
We make this claim rigorous in the following proposition.

\begin{prop}\label{prop:argmax}
Fix $n\ge 2$.
For any $k=1,\dots,n$, we have that
\begin{align}\label{eq:argmax}
I_n^{(k)}=\mathrm{argmax}\big\{\iota(J):J\ &\mathrm{utilizes}\ k\ \mathrm{out}\\ &\mathrm{of}\ n\ \mathrm{dimensions}\big\}.\nonumber
\end{align}
\end{prop}

This result justifies the use of IsoScore for measuring the extent to which a point cloud optimally utilizes all dimensions of the ambient space because it demonstrates that $\iota(I_n^{(k)})$ is the maximal IsoScore for any covariance matrix with $k$ non-zero entries and $n-k$ zero entries.

\begin{proof}[Proof of Proposition~\ref{prop:argmax}]
In this section we let $\Diag^+(n)$ denote the set of $n\times n$ real matrices which vanish away from the diagonal and whose diagonal entries are all non-negative.
The set $\Diag^+(n)$ parameterizes the set of all $n\times n$ sample covariance matrices after performing PCA-reorientation.
We also let $\Diag^+(n,k)\subseteq\Diag^+(n)$ denote that subset whose first $k$ diagonal entries are non zero and whose last $n-k$ diagonal entries are zero.
The set $\Diag^+(n,k)$ parameterizes the set of sample covariance matrices post-PCA reorientation which utilize $k$ out of $n$ dimensions of space. 
Covariance matrices in $\Diag^+(n,k)$ represent point clouds with the property that $\mathrm{Var}(x_i)>0$ for $i=1,\dots,k$, and $\mathrm{Var}(x_i)=0$ for $i=k+1,\dots,n$.

It suffices to show that, for every $J\in\Diag^+(n,k)$, we have that $\iota(J)\le\iota(I_n^{(k)})$, or equivalently, $\delta(J)\ge\delta(I_n^{(k)})$.
Write $\hat I_{n,D}^{(k)}=(\sqrt{n/k},\dots,\sqrt{n/k},0,\dots,0)$ and $J_D=(a_1,\dots,a_k,0,\dots,0)$, where $a_1^2+\cdots a_k^2=n$.
Then we must show that $\|J_D-\mathbf 1\|\ge \|\hat I_{n,D}^{(k)}-\mathbf 1\|$, or equivalently,
\[\sum_{i=1}^k (a_i-1)^2+n-k\ge\sum_{i=1}^k(\sqrt{n/k}-1)^2+n-k.\]
This latter estimate is equivalent to
\[\sum_{i=1}^k a_i\le\sqrt{nk}.\]
By Jensen's inequality, applied with the convex function $f(x)=x^2$, we have that
\[f\left(\sum_{i=1}^k\frac{a_i}{k}\right)\le\sum_{i=1}^k\frac{f(a_i)}{k}.\]
Simplifying, this implies that $(a_1+\cdots+a_k)^2\le kn$.
This completes the proof.
\end{proof}

\section{Numerical Experiments}
\label{appendix:num_experiments}
In this section, we provide explicit details of how each test is designed. We provide code for all experiments at: \href{https://github.com/bcbi-edu/p\_eickhoff\_isoscore}{\textit{https://github.com/bcbi-edu/p\_eickhoff\_isoscore}}.

\begin{enumerate}
    \item \textbf{Test 1: Mean Invariance.} To assess whether the five scores are mean invariant, we start with $100,000$ points sampled from a 10-dimensional multivariate Gaussian distribution with covariance matrix equal to the identity and a common mean vector $M=[\mu,\mu,...,\mu]$. We compute scores for $\mu=0,1,2,...,20$. 
    
    \item \textbf{Test 2: Scalar Invariance.}  We test for the property of scalar invariance by sampling $100,000$ points from a 5D Gaussian distribution with common mean vector $M=[3,3,3,3,3]$ and covariance matrix equal to $\lambda \cdot I_5$. We then compute scores for each point cloud as we increase $\lambda$ from 1 to 25. 
    
    \item \textbf{Test 3: Maximum Variance.} We start by sampling $100,000$ points from a 10D multivariate Gaussian distribution with zero common mean vector and a diagonal covariance matrix with nine entries equal to $1$ and one diagonal entry equal to $x$. In our experimental setup, we compute all five scores as we increase $x$ from 1 to 75.
    
    \item \textbf{Test 4: Rotation Invariance.} Our baseline point cloud $X \subset \mathbb{R}^{n}$ consists of $100,000$ points sampled from a 2D zero-mean Gaussian distribution with a covariance matrix equal to $\begin{psmallmatrix} 1 & 0.8 \\ 0.8 & 1 \end{psmallmatrix}$. We rotate $X$ by $120^{\circ}$ and $240^{\circ}$. Lastly, we project $X$ using PCA reorientation while retaining dimensionality to obtain a point cloud $X^{\text{PCA}}$.

    \item \textbf{Test 5: Dimensions Used (Fraction of Dimensions Used Test).} For our first experiment, which we term the ``fraction of dimensions used test,'' we sample $100,000$ points from a 25D multivariate Gaussian distribution with a zero common mean vector and a diagonal covariance matrix where the first $k$ entries are $1$ and the remaining $n-k$ diagonal elements are $0$. We refer to $k$ as the number of dimensions uniformly used by our data (see Definition~\ref{def:uniform_utilization}). For our experiment we let $k=1,2,3,...,25$, and compute the corresponding scores. 
    
    \item \textbf{Test 5: Dimensions Used (High Dimensional Test).} A good score of spatial utilization should allow for easy comparison between different vector spaces even when the dimensionality of the two spaces is different. We sample $100,000$ points from a zero-mean Gaussian distribution with identity covariance matrix $I_n$ and increase the dimension of the distribution from $n=2,\dots,100$.
    
    \item \textbf{Test 6: Global Stability.} We generate a ``skewered meatball'' by sampling $1,000$  points from a line in 3D space and increase the number of points sampled from a 3-Dimensional, zero-mean, isotropic Gaussian from 0 to $150,000$. 
\end{enumerate}

\section{Geometry of Isotropy}
\label{appendix:geometry_of_isotropy}
\begin{figure}[h!]
    \centering
    \includegraphics[width=3cm]{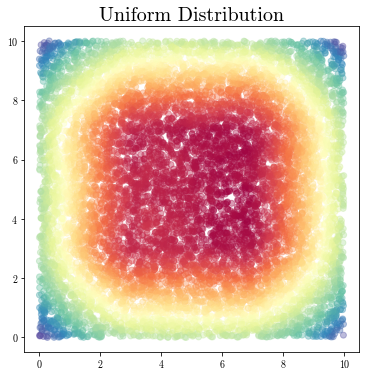}
    \includegraphics[width=3cm]{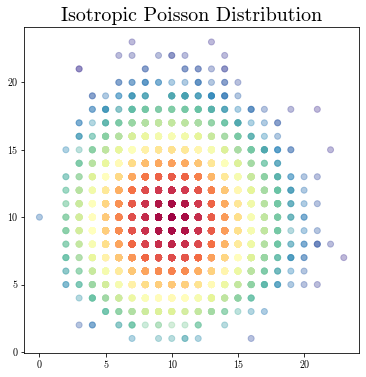}
    \includegraphics[width=3cm]{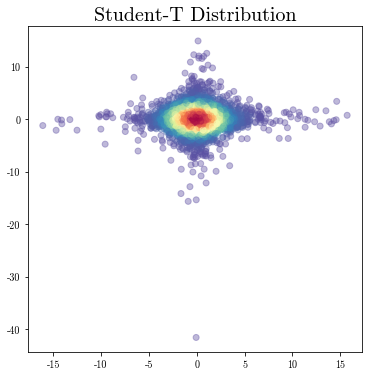}
    \includegraphics[width=3cm]{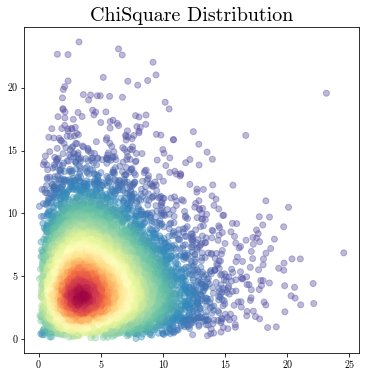}
    \caption{Points sampled from a Uniform distribution, Poisson distribution, Student-T distribution and ChiSquare distribution respectively}
    \label{fig:iso_geo}
\end{figure}

 Each of the distributions illustrated in Figure~\ref{fig:iso_geo} has a covariance matrix proportional to the identity and is therefore maximally isotropic. Namely, the variance is distributed equally in all directions. Despite receiving an IsoScore of 1, the geometry of the point clouds are vastly different. We can only comment on the geometry of the point cloud if the underlying distribution of the space is known.

\end{document}